\documentclass{article}

\usepackage{natbib}
\usepackage[utf8]{inputenc} 
\usepackage[T1]{fontenc}    
\usepackage{lmodern}
\usepackage[hypertexnames=false]{hyperref}       
\usepackage{url}            
\usepackage{booktabs}       
\usepackage{amsfonts}       
\usepackage{nicefrac}       
\usepackage{microtype}      
\usepackage{xcolor}         

\usepackage{graphicx} 
\usepackage[english]{babel}

\usepackage{amsmath}
\usepackage{amssymb}
\usepackage{amsthm}
\usepackage{bbm} 
\usepackage{nicefrac} 

\usepackage[normalem]{ulem} 

\usepackage{amsbsy}
\usepackage{tikz}
\usetikzlibrary{positioning}
\usepackage{stmaryrd}
\usepackage{color}
\usepackage{colortbl} 
\usepackage{makeidx}
\usepackage[shortlabels]{enumitem} 

\usepackage{caption}
\usepackage{subcaption}
\usepackage{arydshln}  
\usepackage{makecell} 
\usepackage{multirow}  

\usepackage{algorithm}
\usepackage{algpseudocode}

\usepackage{cleveref}
\crefname{enumi}{part}{parts}
\usepackage{comment}

\usepackage{geometry}
\usepackage{graphicx}
\usepackage{wrapfig}
\usepackage{mdframed} 
\usepackage{ragged2e}

\newtheorem{thm}{Theorem}[section]
\crefname{thm}{Theorem}{Theorems}
\newtheorem{lemma}{Lemma}[section]
\crefname{lemma}{Lemma}{Lemmas}
\newtheorem{cor}{Corollary}[section]
\crefname{cor}{Corollary}{Corollaries}

\crefname{cond}{Assumption}{Assumptions}

\crefname{mech}{Mechanism}{Mechanisms}

\crefname{algorithm}{Algorithm}{Algorithms}
\crefname{section}{Section}{Sections}
\crefname{defi}{Definition}{Definitions}
\crefname{table}{Table}{Tables}
\crefname{figure}{Figure}{Figures}

\theoremstyle{remark}
\newtheorem{rem}{Remark}[section]
\crefname{rem}{remark}{remarks}

\theoremstyle{definition}
\newtheorem{defi}{Definition}[section]

\newcommand{\argmin}[1]{\underset{#1}{\textnormal{argmin }}}
\newcommand{\E}{\mathbb{E}}

\title{Gaussian Differential Private Bootstrap by Subsampling}
\author{ Holger Dette\\ \texttt{\small holger.dette@ruhr-uni-bochum.de} \and Carina Graw \\ \texttt{\small carina.graw@ruhr-uni-bochum.de} }
\date{}

\begin{document}
	\maketitle
	\begin{abstract}\noindent
		Bootstrap is a common tool for quantifying uncertainty in data analysis. However, besides additional  computational costs in the application of the bootstrap on massive data,  a challenging  problem in  bootstrap based inference
		under Differential Privacy consists in the fact that  it requires repeated access to the data. As a consequence, bootstrap based differentially private inference requires a significant increase of the privacy budget, which on the other hand comes with a substantial loss in statistical accuracy. 
		
		A potential solution to reconcile the conflicting goals of statistical accuracy and privacy is 
		to analyze the data under parametric  model assumptions and in the last decade, several parametric bootstrap methods for inference under privacy have been investigated. However,  uncertainty quantification  by parametric bootstrap is only valid if the
		the quantities of interest can be identified as the parameters of a statistical model and  the 
		imposed model assumptions are (at least approximately) satisfied.
		An alternative to parametric methods is the empirical bootstrap that is a widely used tool for non-parametric inference and  well studied in the non-private regime. 
		However, under privacy, less insight is available. 
		In this paper, we propose a private empirical $m$ out of $n$ bootstrap and validate its consistency and privacy guarantees under Gaussian Differential Privacy. 
		Compared to the the private $n$ out of $n$ bootstrap, our approach has several advantages. First, it comes with less computational costs, in particular for massive data. Second, the proposed procedure needs less additional noise in the bootstrap iterations,
		which leads to an improved statistical accuracy  while asymptotically guaranteeing the same level of privacy. Third, we demonstrate much better finite sample properties compared to the currently available procedures.
	\end{abstract}
	
	\section{Introduction}
	
	In the era of  big data enormous amounts of - possible sensible - information is collected about individuals and over the last decades  privacy concerns became more and more important. 
	In the past, successful de-anonymisation attacks have been performed leaking private, possible sensitive data of individuals \cite[see, for example, ][]{narayanan2006break,gambs2014anonymization,eshun2022two}.
	{\it Differential Privacy} (DP) is a framework that has been introduced by \citet{dwork2006} to protect an individual against such attacks while still being able to learn something about the population.
	Nowadays, Differential Privacy has become a state-of-the-art framework which has been implemented by the US Census and large companies like Apple, Microsoft and Google \cite[see][]{ding2017collecting,abowd2018us}.
	
	Since its introduction, numerous differentially private algorithms have been developed for many applications and we refer the reader to the recent survey papers by \citet{wang2020comprehensive,xiong2020comprehensive,yang2023local} and the references therein. 
	Nowadays, numerous DP tools are available, but from a statistical perspective there is a lack of general techniques for conducting inference under DP such as the uncertainty quantification of DP estimates or the derivation of DP confidence intervals.
	In the non-private setting, bootstrap is a widely used method to address this question, and parametric bootstrap for private statistical inference have been widely studied \citep[see, for example,][and the references therein]{wangdebiased,ferrando2022parametric,awan2023simulation}.
	However, this method relies on strong (parametric) model assumptions and might yield to invalid inference if these assumptions are violated.
	The non-parametric bootstrap is less sensitive but its application for statistical inference under DP is very challenging as it requires multiple, say $B$, queries to the private data.
	As a consequence, it is very difficult to balance the trade-off between privacy and statistical accuracy, and non-parametric bootstrap under DP is not so widely studied.
	\cite{brawner2018bootstrap} considered the  empirical bootstrap under zero concentrated DP and \cite{wang2022} investigated empirical bootstrap under $f$-DP.
	Furthermore, \cite{chadha2024resampling} proposed a DP non-parametric bootstrap procedure called {\it Bag of little Bootstraps} (BLBQuant).
	
	{\bf Our contribution:} In this paper we present a general non-parametric $f$-DP bootstrap procedure that - in comparison to previous approaches - adds less noise to each bootstrap estimator while guaranteeing the same level of privacy. 
	This leads to an increased accuracy of the obtained confidence intervals, especially for high privacy and small sample sizes. 
	Our approach utilizes the fact that sub-sampling can improve privacy guarantees 
	\citep[see, for example,][]{balle2018privacy,koga2022privacy} 
	and is based on the $m$ out of $n$ bootstrap, which draws a subsample of $m < n$ observations with replacement from the full sample of $n$ observations. 
	On the one hand, this can be used to decrease the noise due to privatization, and on the other hand, it can significantly reduce computational time.
	We validate the statistical  consistency of the $m$ out of $n$ bootstrap and provide a privacy guarantee.
	Furthermore, we investigate the finite sample performance of this new approach by means of a simulation study. Compared to previous work our approach has (at least) three advantages.   First, it comes with much less computational costs, in particular for massive data. Second, the proposed procedure needs less additional noise in the bootstrap iterations,
	which leads to an improved statistical accuracy  while asymptotically guaranteeing the same level of privacy. Third, we demonstrate much better finite sample properties  (with respect to coverage and length of confidence intervals)
	compared to the currently available procedures.
	
	{\bf Structure of the paper:} In \cref{sec:basics} we give a brief overview of $f$-Differential Privacy and the empirical bootstrap. 
	In \cref{sec:mn-bootstrap} we present the $m$ out of $n$ bootstrap and state its consistency and privacy guarantees. 
	Finally, we evaluate the new procedure in \cref{sec:sim} by means of a simulation study while proofs and additional material can be found in the appendix.
	
	\section{Background} \label{sec:basics}
	
	We start by giving a short overview on $f$-Differential Privacy  as introduced  in  \cite{dong2022gaussian} and the bootstrap under Gaussian Differential Privacy proposed by \cite{wang2022}, which is important for our work.
	
	\subsection{$f$-Differential Privacy}
	Let $(\mathcal{X},\mathbb{X})$, $(\mathcal{Y},\mathbb{Y})$ be measurable spaces. 
	We call a randomized mapping $M: \mathcal{X}^n \to \mathcal{Y}$ a mechanism that maps a dataset $S\in \mathcal{X}^n$ onto a random variable $M(S)$. 
	With abuse of notation we will denote its distribution by $M(S)$ as well.
	We call two data sets $S$, $S'\in \mathcal{X}^n$ neighboring if they differ in exactly one entry. 
	A first notion of privacy, \textit{$(\epsilon,\delta)$-Differential Privacy}, was introduced by \cite{dwork2006}:
	\begin{defi}
		\label{def:f-dp}
		A mechanism $M$ is $(\epsilon,\delta)$-{\it differentially private} (DP), if
		\begin{equation*}
			P(M(S)\in\mathcal{S})\leq e^{-\epsilon} P(M(S')\in\mathcal{S}) + \delta
		\end{equation*}
		for all events $\mathcal{S}\in \mathbb{Y}$ and for all neighboring datasets $S,S'\in \mathcal{X}^n$ .
	\end{defi}
	However, the meaning of the parameters $\epsilon,\delta$ and their interpretation is not obvious \citep[see][for a statistical interpretation in terms of testing]{wasserman2010statistical}. \textit{$f$-Differential Privacy} as  recently  introduced by  \cite{dong2022gaussian} 
	is an alternative definition of privacy, that is more natural from a statistical point of view. It is given in terms of type I and type II errors of the testing problem whether the input of the mechanism $M$ was $S$ or $S'$.
	\begin{defi}
		For two probability distributions $P$ and $Q$ on the same space, define the trade-off function $T(P,Q):[0,1]\to[0,1]$ as
		\[T(P,Q)(\alpha)=\inf\{\beta_\phi: \alpha_\phi\leq \alpha\}, \]
		where the infimum is taken over all (measurable)     rejection rules for the hypotheses $H_0: P$  versus  $H_1: Q$ and $\alpha_\phi$ ($\beta_\phi$) denotes the type I (type II) error of a rejection rule $\phi$.\\
		Let $f: [0, 1]  \to  [0, 1] $ be a trade-off function. A mechanism $M$ is said to be $f${\it -differentially private} ($f$-DP) if the inequality  
		\[
		T(M(S),M(S'))\geq f
		\]
		holds 
		for all neighboring datasets $S,S'\in \mathcal{X}^n$ (i.e. datasets that differ in one entry). Here, the inequality   $g\geq h$ for two functions $g,h:[0,1]\to[0,1]$ means that $g(x)\geq h(x)$ for all $x\in [0,1]$.\\
		A mechanism $M$ is said to satisfy $f$-DP for groups of size $k$, if
		\[T(M(S),M(S'))\geq f\]
		for all $k$-neighboring datasets $S$ and $S'$ (i.e. datasets that differ in $k$ entries).\\
	\end{defi}
	
	A special case of $f$-DP is {\it Gaussian Differential Privacy} (GDP), where the function $f$ in \cref{def:f-dp} is the trade off function between two normal distributions with variance $1$, one with mean zero and the other  with mean $\mu>0$. 
	Denoting a normal distribution with mean $\mu$ and variance $\sigma^2$ by $N(\mu,\sigma^2)$,
	$\mu$-GDP means that distinguishing the output of a mechanism based on two neighboring databases is as least as hard as distinguishing the standard normal distribution $N(0,1)$ from $N(\mu,1)$.
	A straightforward calculation shows that the trade-off function  $G_\mu:[0,1]\to[0,1]$ of a $N(0,1)$ and $N(\mu,1)$ distribution is given by 
	\begin{align}
		\label{det102}
		G_\mu(\alpha) = T(N(0,1),N(\mu,1))(\alpha)=\Phi(\Phi^{-1}(1-\alpha)-\mu) , 
	\end{align}
	where  $\Phi$ denotes the CDF of the standard normal distribution.
	
	\begin{defi} [$\mu$-Gaussian Differential Privacy]\label{def:gaussianDP}
		A mechanism $M$ is said to satisfy $\mu$-Gaussian Differential Privacy ($\mu$-GDP) if it is $G_\mu$-DP, that is 
		\[T(M(S),M(S'))\geq G_\mu\]
		for all neighboring datasets $S$ and $S'$.
	\end{defi}
	The next result shows that the Gaussian Mechanism is $\mu$-GDP.
	\begin{thm}[Gaussian Mechanism] \label{thm:gauss_mech}
		Let $\theta(S_n)$ be a univariate statistic computed on a dataset $S_n$ of size $n$ and define the sensitivity of $\theta$ as
		\begin{equation} \label{eq:def-delta}   \Delta=\Delta_\theta(n)=\sup_{S_n,S'_n neighboring} \lvert \theta(S_n)-\theta(S'_n)\rvert .
		\end{equation}
		Let $\xi \sim N(0, \sigma_\mu^2)$, where  $\sigma_\mu^2=\Delta^2 / \mu^2$, then  the Gaussian mechanism  \begin{equation}\label{mech:gaussian-mech}
			M_{Gauss}(S_n)=\theta(S_n) + \xi,
		\end{equation} 
		is $\mu$-GDP.
	\end{thm}
	Similar to $(\epsilon,\delta)$-DP mechanisms, 
	$f$-DP mechanisms,  have keen properties under post processing, composition and extension to group privacy, which are summarized in the following statements. 
	\begin{thm} \label{thm:properties_f_DP}
		~~
		\begin{enumerate}
			\item If a mechanism $M$ is $f$-DP, then its post-processing $Proc\circ M$ is also $f$-DP.
			\item Define the tensor product of two trade-off functions $f=T(P,Q)$, $g=T(P',Q')$ as
			\[
			f\otimes g := T(P\times P', Q\times Q') ,
			\]
			where $P\times P'$ denotes the product measure of the distributions $P$, $ P'$.
			Let $M_i(\cdot,y_1,\ldots, y_{i-1}):\mathcal{X}^n\to \mathcal{Y}_i$ be $f_i$-DP for all $y_1\in \mathcal{Y}_1,\ldots, y_{i-1}\in \mathcal{Y}_{i-1}$. Then the $n$-fold composed mechanism $M:\mathcal{X}^n\to \mathcal{Y}_1\times \ldots \times\mathcal{Y}_n$ is $f_1\otimes\ldots \otimes f_n$-DP. \\
			In particular, if $f_i=G_{\mu_i}$ for $i=1,\ldots, n$, the composed mechanism is $\mu$-GDP, where $\mu=\sqrt{\mu_1^2+\ldots +\mu_n^2}$.
			\item If a mechanism $M$ is $\mu$-GDP, then it is $k\mu$-GDP for groups of size $k$.
		\end{enumerate}
	\end{thm}
	
	There is an equivalence between $\mu$-GDP and the widely used notion of $(\epsilon,\delta)$-DP:
	
	\begin{lemma}
		\label{lem:dp_to_gdp}
		A mechanism is $\mu$-GDP if and only if it is $(\epsilon, \delta(\epsilon))$-DP for all $\epsilon\geq 0$, where
		\[
		\delta(\epsilon,\mu)=\delta(\epsilon)=\Phi(-\epsilon/\mu +\mu/2) - e^\epsilon\Phi(-\epsilon/\mu - \mu/2) .
		\]
	\end{lemma}
	\subsection{$f$-DP Bootstrap}
	
	In order to provide statistical guarantees for  DP estimators, one usually assumes that the available data points $x_1, \ldots x_n$  are realizations of random variables $X_1, \ldots X_n$.
	For the sake of simplicity,  let $X_1,\ldots,X_n\sim G$ be independent and identically distributed random variables following a distribution $G$. 
	Denote by $\delta_x$ the Dirac measure at the point $x$ and by  
	\begin{align}
		\label{det1}
		G_n  = \frac{1}{n} \sum_{i=1}^n  \delta_{X_i} 
	\end{align}
	the  empirical distribution of $X_1, \ldots,  X_n$. In other words: the the distribution $G_n$ puts masses $\frac{1}{n} $ at the points $X_1, \ldots , X_n$.
	For a parameter of interest  $\theta=\theta(G)$ depending on the underlying distribution $G$ 
	a  natural estimator  is then given by 
	\begin{align}
		\label{det2}
		\hat{\theta}_n:=\theta(X_1, \ldots,X_n):=\theta(G_n)
	\end{align} 
	(a typical example in the real case is the mean $\theta (G) = \int_{\mathbb{R}} x dG(x) =\E[X]$ with estimator $\hat{\theta}_n = \theta(G_n) = \int_{\mathbb{R}} x dG_n(x) = \frac{1}{n} \sum_{i=1}^n X_i$).
	Note that  $\hat{\theta}_n$ is a random variable with a distribution depending  on the unknown underlying distribution $G$. For statistical inference regarding the parameter $\theta (G) $ the (unknown) distribution of  the statistic 
	$\sqrt{n}(\theta(G_n)-\theta(G))$ is of interest. 
	For example, if this  distribution would be known and  $q_\alpha$ denotes its $\alpha$ quantile, then 
	\begin{align}
		\label{eqn:CIs}
		\Big[ \theta(G_n) -\frac{1}{\sqrt{n}}q_{1-\alpha}, \theta(G_n)-\frac{1}{\sqrt{n}}q_{\alpha}\Big ] 
	\end{align}
	defines an $(1-2\alpha)$-confidence interval for the parameter $\theta(G)$.
	However, in general this distribution is not known and the interval $\eqref{eqn:CIs}$ cannot be used in practice. 
	In many cases the distribution of the statistic $\sqrt{n}(\theta(G_n)-\theta(G))$ can be approximated by a known limiting distribution for large sample size $n$ (this is actually the reason why we use the scaling factor $\sqrt{n}$ in \eqref{eqn:CIs}).
	However, this limiting distribution depends on further unknown parameters and it has to be worked out case by case.
	\\
	Bootstrap is a widely used principle for uncertainty quantification in  statistical inference avoidng these problems. 
	The basic idea  is the following: assume that a known distribution $G^\star$ is - in some sense - close to the unknown distribution $G$, then the distribution of the statistic $T_n:= \sqrt{n}(\theta(G_n)-\theta(G))$ is close to the one of $T_n^*:=\sqrt{n}(\theta(G_n^\star)-\theta(G^\star))$, where $G_n^\star$ denotes the empirical distribution 
	based on a so called bootstrap sample $X_1^*, \ldots , X_n^* \sim G^\star$. 
	Since $G^\star$ is known one can  generate such a sample by simulation to obtain a draw from the distribution of  
	$T_n^\star$.
	By repeating this procedure  $B$ times we obtain $B$  (independent) bootstrap replications  $T_{n,1}^\star , \ldots , T_{n,B}^\star$, where $T_{n,b}^\star =\sqrt{n}(\theta^*_{n,b} -\theta(G^\star)) $ and $\theta^*_{n,b}$ denotes the estimator of $\theta$ in the $b$-th run. By the law of large numbers,
	\begin{align} \label{det3a} 
		\frac{1}{B} 
		\sum_{b=1}^B 
		I \{ T_{n,b}^\star \leq x \} \approx  \mathbb{P} (T_n^* \leq x) \approx  \mathbb{P} (T_n \leq x)  ~
	\end{align}
	and consequently we can approximate  the  $\alpha$-quantile of the distribution of $T_n$ by the  empirical $\alpha$-quantile of the sample $T_{n,1}^\star , \ldots , T_{n,B}^\star$. 
	Using this quantile in \eqref{eqn:CIs}
	would yield to a valid $(1-2\alpha )$ confidence interval for the parameter $\theta (G)$. \\
	To make this work in practice, three issues have to be resolved. 
	First, one has to find an estimator for the unknown distribution $G^\star$ and we use the empirical distribution function $G_n$ in \eqref{det1} for this purpose. 
	Next, one has to prove that the second approximation in \eqref{det3a} is valid (for large $n$).
	Note that  drawing a sample 
	$X_1^\star , \ldots , X_n^\star $ from $G_n$ means  sampling uniformly with replacement from the original sample $X_1,\ldots,X_n$.
	And at last, we have to generate a sufficiently 
	large number of bootstrap replications such that  the first approximation in \eqref{det3a} is  valid. By the law of large numbers it is always valid if $B $ is chosen sufficiently large, but in practice computational costs have to be taken into account.
	For the sake of completeness, the empirical bootstrap is stated in \cref{algo:emprircal-bootstrap}.
	
	\begin{algorithm}[ht]
		\begin{algorithmic}
			\State{{\bf Input:} } Data $X_1,\ldots,X_n$, estimator $\theta(\cdot)$, number of bootstrap replications $B$
			\State{{\bf Output:} }  estimator $\hat{\theta}_n$, bootstrap replications $T_{n,1}^\star,\ldots,T_{n,B}^\star$ 
			\State Compute $\hat{\theta}_n=\theta(X_1,\ldots,X_n)$
			\For{$b=1,\ldots,B$}
			\State Draw $X_1^\star,\ldots,X_n^\star$ uniformly with replacement from $X_1,\ldots,X_n$
			\State Compute $\theta_{n,b}^\star= \theta(X_1^\star,\ldots,X_n^\star)$
			\State Compute $T_{n,b}^\star =\sqrt{n}(\theta_{n,b}^\star - \hat{\theta}_n)$
			\EndFor
		\end{algorithmic}
		\caption{Empirical Bootstrap}
		\label{algo:emprircal-bootstrap}
	\end{algorithm}
	
	Under privacy constraints, the empirical bootstrap becomes very challenging, as it not only incurs computational costs for the resampling but also requires repeated access to the data.
	Even for the simple case that an estimator is computed on one bootstrap sample that is obtained by drawing $n$ times uniformly with replacement from the original dataset, privacy is a serious issue.
	In the worst case, one data point can be included $n$ times in the new sample, resulting in a constant sensitivity in \eqref{eq:def-delta} of the estimator from the bootstrap sample that does not decrease with the sample size. 
	For most privacy mechanisms, such as the Gaussian mechanism \eqref{mech:gaussian-mech}, the additional inaccuracy of the estimator due to privacy scales with the sensitivity.
	A constant sensitivity implies that the accuracy of a private estimator from the bootstrap sample does not increase with increasing sample size, which is not desirable.
	However, the event that a single data point is included often in the new sample only occurs with small probability, in particular if the sample size $n$ is large. 
\cite{dong2022gaussian} utilize this observation to investigate how the $f$-DP guarantee is affected by one resampling procedure with replacement. 
\cite{wang2022} expand their analysis to multiple resampling iterations $B$ and propose a $\mu$-GDP bootstrap procedure, where each bootstrap estimator is privatized by the Gaussian Mechanism \eqref{mech:gaussian-mech} with an adjusted privacy budget $\mu_B$.
While in the non-private case the number of bootstrap samples $B$ can be chosen arbitrarily large, this is not the case anymore under privacy constraints since $\mu_B$ decreases with the rate $1/\sqrt{B}$ and, as a direct consequence, the variance in the Gaussian mechanism grows linearly  with $B$.
Therefore, an increased number of bootstrap samples $B$ leads to a larger error of the bootstrap estimators due to privacy and as a consequence  to a less accurate approximation of the distribution of interest.
On the other hand, choosing $B$ too small leads to a high inaccuracy of the bootstrap caused by a poor approximation by  the empirical distribution of the bootstrap estimator (see the first approximation in \eqref{det3a}).
As a rule of thumb, \cite{wang2022} proposed to choose $B$ proportional to $n\mu^2$, that is
\begin{align}
	\label{det6}
	B=\Theta(n\mu^2)
\end{align}
to get the best of both worlds.

\section{$m$ out of $n$ Bootstrap}
\label{sec:mn-bootstrap}
Especially for moderate  samples and  high privacy (i.e. small $\mu$) requirements, the rule of thumb \eqref{det6}   
would result in a rather small number of bootstrap iterations $B$.
For example, if $\mu=0.1$ and  the sample size $n=1000$ one obtains  
for the number of bootstrap replications $B=10$, which is not sufficient to compute empirical $5\%$ quantiles with good accuracy.
In this section, we  attempt to reduce the additional variance due to privacy in the bootstrap iterations without affecting the privacy guarantees. 
Our approach is based  the $m$ out of $n$ bootstrap, where  the private  estimators are computed only from 
a subsample of bootstrap observations of size $m$, which is substantially smaller than the size $n$ of the original sample.
The $m$ out of $n$ bootstrap in the non-private case is introduced in \cite{politis1994large}.
Under the additionally requirement $m=o(n)$, it is often consistent in scenarios where the $n$ out of $n$ bootstrap fails (see, for example, \cite{bickel1997resampling}).
Furthermore, computational time is reduced for the $m$ out of $n$ bootstrap, since each bootstrap estimator is only computed from a subsample of size $m < n$.
In the context of  DP a 
smaller size of the bootstrap sample decreases the probability of an individual being included in the bootstrap sample and therefore less noise has to be added to guarantee the same level of privacy. 
As a consequence one can choose a larger number $B$ of bootstrap iterations to decrease the error in the first approximation in \eqref{det3a}.
However,  while a larger $B$ improves the accuracy of this approximation, it increases the additional variance in  the estimator from each bootstrap sample, leading to a less accurate approximation of the distribution of interest (see the second approximation in \eqref{det3a}). 
\smallskip

We now investigate  the $m$ out of $n$ bootstrap procedure and the trade-off between the choice of  $B$ and $m$
in more detail.
For this purpose, let $\hat{\theta}_n = \theta(G_n)$ denote the estimate \eqref{det2} of  
the parameter $\theta=\theta(G)$ from the sample  $X_1,\ldots, X_n\sim G$, where    $G_n$ is the empirical distribution function defined in \eqref{det1}. 
We assume that the sensitivity \eqref{eq:def-delta} of $\hat{\theta}_n$ is given by  $\Delta_\theta(n)= l/n$  for some constant $l>0$. 
Then, by \cref{thm:gauss_mech},
a $\mu$-GDP estimator $\bar{\theta}_n$ for $\theta$ is given by the Gaussian mechanism \eqref{mech:gaussian-mech}, that is 
\begin{equation}\label{eqn:gdp_esti}
	\bar{\theta}_n = \hat{\theta}_n + Z ~,
\end{equation}
where $Z\sim N\left(0,\sigma_\mu^2\right)$ and 
$\sigma_\mu^2=
\frac{l^2}{n^2\mu^2}$.
The distribution of this estimator is approximated by the $m$ out of $n$ bootstrap, which draws  $m$ data points  uniformly with replacement from the original sample, that is $X_1,^\star,\ldots,X_m^\star\sim G_n$. We define the corresponding  empirical distribution  by $G_m^\star = \frac{1}{m} \sum_{i=1}^m \delta_{X_i^\star} $ and consider the  estimator 
\begin{align}
	\hat{\theta}^{\star}_m &= \theta(X_1^\star,\ldots,X_m^\star)=\theta(G_m^\star) \label{eqn:bootstrap_np}
\end{align} 
for $\theta (G) $ from the sample $X_1^\star,\ldots,X_m^\star$.
We use again the Gaussian mechanism to obtain a privatized version of  $\hat{\theta}^{\star}_m$, that is
\begin{align}
	\bar{\theta}^{\star}_m &= \hat{\theta}^{\star}_m+ Y ~, \label{eqn:bootstrap_private}
\end{align}
where the random variable $Y \sim N(0,\sigma_{m,B}^2)$ is independent of $X_1,\ldots,X_n$ and 
\begin{align}
	\label{det5}
	\sigma_{m,B}^2 &=B(1-(1-1/n)^m)\frac{n+m-1}{n} \frac{l^2}{mn\mu^2}=\frac{\Delta_\theta^2(m)}{\mu_B^2(m,n)}
\end{align}
with
\begin{align}
	\mu_B^\star&=\mu_B(m,n)= \frac{\mu}{\sqrt{B\left(1-\left(1-\frac{1}{n}\right)^m\right)\left(\frac{n+m-1}{n}\right)\frac{m}{n}}} ~.
	\label{det3}
\end{align} 
We now  approximate the unknown distribution of the statistic $\bar T_n =\sqrt{n}(\bar{\theta}_n-\theta)$ by the distribution $ \bar T_m^*=\sqrt{m}(\bar{\theta}_m^\star-\bar{\theta}_n)$ of the privatized  bootstrap estimator.
This procedure is summarized in \cref{algo:mn_dp_bootstrap}, which 
gives a sample $\bar T_{m,1}^* , \ldots , \bar T_{m,B}^*$ (of size $B$) from  the bootstrap distribution, where $\bar T_{m,b}^* =  \sqrt{m}(\bar{\theta}^{\star}_{m,b}-\bar{\theta}_n)$. 
If ${q}_{1-\alpha}^\star$ denotes the empirical $(1-\alpha )$-quantile of the bootstrap sample $\bar{T}_{m,1}^\star,\ldots,\bar{T}_{m,B}^\star$, we define the (bootstrap) confidence interval for the parameter $ \theta (G)$ by 
\begin{align}\label{eqn:CI_mn}
	C_{m,n}(\alpha)& =\Big [\bar{\theta}_n - \frac{1}{\sqrt{n}}{q}_{1-\alpha}^\star,   \bar{\theta}_n - \frac{1}{\sqrt{n}}q^\star_{\alpha}\Big ] ~.
\end{align}
In the following discussion we will show that this interval  preserves  (for large $B$) $2\mu$-GDP  and has (for large $n,m;$ $m=O(n)$) the correct coverage $(1-2\alpha)$.
Note that for $m=n$ we obtain the bootstrap proposed by \cite{wang2022},  but based on our empirical results presented in Section \ref{sec:sim} we strongly recommend to use a bootstrap sample of a much smaller size than $n$. 
\\

\begin{algorithm}[ht]
	\begin{algorithmic}
		\State {\bf Input:} Data $X_1,\ldots,X_n$, estimator $\theta(\cdot)$ and its sensitivity $\Delta_\theta(\cdot)$, number of bootstrap replications $B$, privacy parameter $\mu$, size of the bootstrap sample $m$
		\State {\bf Output:} $\sqrt{2}\mu$-GDP of the estimator $\bar{\theta}_n$ and of $\bar{T}^\star_{m,1},\ldots,\bar{T}^\star_{m,B}$ 
		\State Compute the quantity  $\mu_B^*=\mu_B(m,n) $ defined in \eqref{det3} 
		\State Draw $ Z\sim N(0,\nicefrac{\Delta_\theta^2(n)}{\mu^2})$
		\State Compute $\bar{\theta}_n=\theta(X_1,\ldots,X_n) + Z$
		\For{$b=1\ldots,B$}
		\State Draw $X_1^\star,\ldots,X_m^\star$ uniformly from $X_1,\ldots,X_n$
		\State Draw $Y_b \sim N\left(0,\frac{\Delta_\theta^2(m)}{{\mu_B^\star}^2}\right)$
		\State Compute $\bar{\theta}^\star_{m,b}=\theta(X_1^\star,\ldots,X_m^\star) + Y_b$
		\State Compute $\bar{T}_{m,b}=\sqrt{m}(\bar{\theta}^\star_{m,b} - \bar{\theta}_n)$
		\EndFor
	\end{algorithmic}
	\caption{$\mu$-GDP $m$ out of $n$ Bootstrap }
	\label{algo:mn_dp_bootstrap}
\end{algorithm}

\paragraph{Privacy Guarantee.}
The privacy guarantee of \cref{algo:mn_dp_bootstrap} is given by the $B$-fold composition of the privacy guarantees of the bootstrap estimator. 
To obtain the privacy guarantee of one bootstrap estimator, consider two neighboring data sets that differ in $X_1$.
Then, $\bar{\theta}_{m}^\star$ in \eqref{eqn:bootstrap_private} is $i\mu_B^\star$-GDP if $X_1$ is included $i$ times in the bootstrap sample (note that, by \cref{thm:properties_f_DP},  a $\mu$-GDP mechanism is $k\mu$-GDP for groups of size $k$).
Denoting the probability of this event by $p_{m,i}=\binom{m}{i}\left(\frac{1}{n}\right)^i \left(1-\frac{1}{n}\right)^{m-i}$, the computation of one bootstrap estimator $\bar{\theta}_{m,b}^\star$ is given by a mixture of mechanisms, i.e. with probability $p_{m,i}$ the computation is $i\mu_B^\star$-GDP.
Therefore, the overall privacy guarantee of the $b$th bootstrap estimator is  
\begin{align}
	\label{det101}
	f_{B,b,boot_{m,n}} = C_{1-p_0}\Big (mix \Big (\Big  (\frac{p_{m,1}}{1-p_0},\ldots,\frac{p_{m,m}}{1-p_0}\Big ),(G_{\mu_B^\star},G_{2\mu_B^\star},\ldots,G_{m\mu_B^\star})\Big  )\Big )
\end{align}
where $ G_\mu $ is the tradeoff function defined in \eqref{det102},  
the function $mix(p,f)$ is the mixture of trade-off functions defined in \eqref{eqn:def-mix} in the appendix 
and $C_{1-p_0}$ is defined in \eqref{eqn:def-C_p} in the appendix. 
This privacy guarantee follows from \cref{thm:privacy_bootstrap_mn} in the Appendix.
By Theorem \ref{thm:properties_f_DP} the privacy of Algorithm \ref{algo:mn_dp_bootstrap} is given by $f_{B,1,boot_{m,n}}\otimes\ldots\otimes f_{B,B,boot_{m,n}}$.

\begin{thm} \label{thm:asymp_privacy_bootstrap_mn}
	Let $\mu \in (0,\infty)$ be a given constant, $\mu_B^\star$ as in \eqref{det3} and for $b=1,\ldots B$ let 
	$f_{B,b,boot_{m,n}}$ be as in \eqref{det101}.
	Then the privacy guarantee of \cref{algo:mn_dp_bootstrap} satisfies
	\[\lim_{B\to\infty}f_{B,1,boot_{m,n}}\otimes\ldots\otimes f_{B,B,boot_{m,n}} \geq G_\mu.\]
\end{thm}

\begin{rem} ~~~
	\begin{itemize}
		\item[(1)] 
		\cref{thm:asymp_privacy_bootstrap_mn} can be generalized to sequences $\{\mu_B^\star\}_{B=1}^\infty$ with 
		$\mu_B^\star\in (0,\infty)$ and $$
		\lim_{B\to\infty} \mu_B^\star \sqrt{B}\sqrt{\left(1-\left(1-\frac{1}{n}\right)^m\right)\frac{m}{n}\left(1+\frac{m-1}{n}\right)}\to \mu.
		$$
		\item[(2)]  
		Note that $\mu_B^\star>\mu$ for some choices of $B$, $m$ and $n$. 
		It might seem counterintuitive that a composition of less private mechanisms actually results in a higher privacy level.
		However, this occurs from using smaller bootstrap samples which leads to perfect privacy with high probability.
		To see this, consider the bootstrap estimator $\hat{\theta}^\star_m$ in \eqref{eqn:bootstrap_np}, where the bootstrap sample is obtained from two neighboring datasets $S$ and $S'$ that differ in the first entry. 
		Given that this first entry is not included in the bootstrap sample, the distribution of $\hat{\theta}^\star_m$ is the same on both data sets - hence there is perfect privacy.
		Denote by $p_0=(1-1/n)^m$ the probability of this event and note that $p_0$ is close to one for suitable $m$ and $n$.
		Therefore, the computation of one bootstrap sample ensures  perfect privacy with high probability and the $B$-fold composition of almost perfectly private mechanisms then converge to the stated $\mu$-GDP limit.
	\end{itemize}
\end{rem}

\paragraph{Consistency.} For the private $m$ out of $n$ bootstrap to be consistent, we basically require two assumptions.
First, the non-private $m$ out of $n$ bootstrap needs to be consistent, that is
\begin{align}
	\label{det4}
	\sup_x |P(\sqrt{n}(\hat{\theta}_n-\theta)\leq x) - P(\sqrt{m}(\hat{\theta}_m^\star-\hat{\theta}_n)\leq x|X_1,\ldots,X_n)|\overset{P}{\to} 0
\end{align}
as $n\to \infty$ and $m=m(n)\to \infty$.
Here 
$P(A|X_1,\ldots,X_n) $
denotes the conditional probability of the event $A$ given $X_1, \ldots , X_n$ (note that $P(A|X_1,\ldots,X_n)$ is a random variable) and  the symbol
${\cal A}= P(A|X_1,\ldots,X_n) \overset{P}{\to} 0$    means convergence in probability, that is  $\lim_{n\to \infty} P \big ( {\cal A} > t \big )=0$ for any $t>0$.   This is a very mild assumption and satisfied for  many statistics  of interest, see for example \cite{bickel1997resampling}.
Second, the noise added due to privacy converges to zero in probability fast enough.
This assumption is satisfied for   the Gaussian mechanism if  the statistic $\hat \theta_n$  has a sensitivity of order $1/n$.

\begin{thm}\label{thm:asymp_consistency_bootstrap_mn}
	Let $\hat{\theta}_n$ be a test statistic for which the non-private $m$ out of $n$ bootstrap 
	is consistent, as specified by \eqref{det4}.
	Further assume that $m=O(n)$, that the limiting distribution of $\sqrt{n}(\hat{\theta}_n-\theta)$ is continuous and that the sensitivity of $\hat{\theta}_n$  satisfies  $\Delta_\theta(n)=O(1/n)$.
	If $B=O(n)$ and $B\to \infty$ then  the $\sqrt{2}\mu$-GDP $m$ out of $n$ bootstrap 
	defined in \cref{algo:mn_dp_bootstrap}
	is consistent, that is 
	\[\sup_x \left\lvert P(\sqrt{n}(\bar{\theta}_n-\theta)\leq x) - \frac{1}{B}\sum_{i=1}^B \mathbb{I}\{\sqrt{m}(\bar{\theta}_m^{\star(i)}-\bar{\theta}_n)\leq x\}\right\rvert\overset{P}{\to} 0.\]
\end{thm}

\begin{cor}\label{cor:CI-correct}
	If  the assumptions of \cref{thm:asymp_consistency_bootstrap_mn}  are satisfied, the interval $C_{m,n}(\alpha)$ defined in \eqref{eqn:CI_mn} is asymptotically $\sqrt{2}\mu$-GDP and it is an asymptotic $(1-2\alpha)$-confidence interval, that is 
	\begin{equation*}
		\lim_{n\to \infty} P(\theta \in C_{m,n}(\alpha)) = 1-2\alpha ~.
	\end{equation*}
\end{cor}

\paragraph{Choice of $m$ and $B$.}
We provide a rule of thumb for the choice of $m$ and $B$ by analyzing the length of the confidence interval in \eqref{eqn:CI_mn}.
In many cases, the distribution of  $\sqrt{n}(\hat{\theta}_n-\theta)$ can be approximated by a normal distribution with variance $\sigma_\theta^2$ for some $\sigma_\theta^2>0$. In this case, the $\mu$-GDP estimator in \eqref{eqn:gdp_esti} is approximately normal distributed as well, i.e.
\[\sqrt{n}(\bar{\theta}_n-\theta)\overset{approx}{\sim} N(0,\sigma_\theta^2 + n\sigma_\mu^2) ~.\]
Therefore, an asymptotic $(1-2\alpha)$-CI for $\theta$ using $\bar{\theta}_n$ is given by 
\[C(\alpha)=\left[\bar{\theta}_n - q_{1-\alpha}\frac{\sqrt{\sigma_\theta^2 + n\sigma_\mu^2}}{\sqrt{n}}, 
\bar{\theta}_n - q_{\alpha}\frac{\sqrt{\sigma_\theta^2 + n\sigma_\mu^2}}{\sqrt{n}}\right]
\]
where $q_\alpha$ denotes the $\alpha$ quantile of a standard normal distribution.
Using the expansion $\sqrt{1+x}=1+\frac{x}{2} + O(x^2)$ for $x=o(1)$, its length is of order $O(\frac{\sigma_\theta + n\sigma_\mu^2/(2\sigma_\theta)}{\sqrt{n}})$.
By \eqref{det4} it follows 
that, given $X_1,\ldots, X_n,Z$, the conditional distribution of the bootstrap statistic  $ \sqrt{m}(\bar{\theta}_m^\star - \bar{\theta}_n)$
can be approximated by  a normal distribution, as well, i.e.
\[
\sqrt{m}(\bar{\theta}_m^\star - \bar{\theta}_n)=\sqrt{m}(\bar{\theta}_m^\star - \hat{\theta}_n - Z)\overset{approx}{\sim } N(-\sqrt{m}Z, \sigma_\theta^2 + m\sigma_{m,B}^2) ~,
\]
where $\sigma_{m,B}^2$ is defined in \eqref{det5}.
Therefore, the confidence interval \eqref{eqn:CI_mn} is approximately equal to 
\begin{align*}
	C_{m,n}(\alpha)\approx 
	&\left[\bar{\theta}_n - \frac{-\sqrt{m}Z +q_{1-\alpha}\sqrt{\sigma_\theta^2 + m\sigma_{m,B}^2}}{\sqrt{n}}, 
	\bar{\theta}_n - \frac{-\sqrt{m}Z+q_{\alpha}\sqrt{\sigma_\theta^2 + m\sigma_{m,B}^2}}{\sqrt{n}}\right].
\end{align*}
Its length is of order $O(\frac{\sigma_\theta+m\sigma_{m,B}^2/(2\sigma_\theta)}{\sqrt{n}})$. 
Plugging in the definitions of $\sigma_\mu^2$ and $\sigma_{m,B}^2$, we obtain the following lengths of the different CIs (up to order $\frac{1}{\sqrt{n}n}$):
\\
\begin{table}[H]
	\centering
	\begin{tabular}{l | c| c | c}
		CI: & asymptotic normal & $n$ out of $n$ bootstrap & $m$ out of $n$ bootstrap \\ 
		& $C(\alpha)$& $C_{n,n}(\alpha)$& $C_{m,n}(\alpha)$ \\     \hline 
		&&&   \\
		length: &$\dfrac{\sigma_\theta}{\sqrt{n}} + \dfrac{1}{2\sigma_\theta}\dfrac{l^2 }{\sqrt{n}n\mu^2}$ 
		& $\dfrac{\sigma_\theta}{\sqrt{n}} + \dfrac{1}{2\sigma_\theta}\dfrac{2(1-(1-\frac{1}{n})^n)Bl^2}{\sqrt{n}n\mu^2}$ 
		&$ \dfrac{\sigma_\theta}{\sqrt{n}} + \dfrac{1}{2\sigma_\theta}\dfrac{(1-(1-\frac{1}{n})^m)Bl^2}{\sqrt{n}n\mu^2} $
	\end{tabular}
\end{table}
\noindent The lengths of the different confidence intervals differ by the second order term.
For the confidence interval obtained by the $n$ out of $n$ bootstrap we see that its length is always larger than the one obtained from the asymptotic normal distribution, since $2B(1-(1-1/n)^n)>1$ for reasonable choices of $B$. 
This is different for the $m$ out of $n$ bootstrap: choosing
\begin{equation} \label{eqn:choose_m}
	m=\frac{\log(1-1/B)}{\log(1-1/n)}\approx \frac{n}{B}
\end{equation}
(or, equivalently, $B=1/(1-(1-1/n)^m)$ ) leads to the same rate of the second term in the length of the $\mu$-GDP asymptotic confidence interval and the confidence intervals obtained by $m$ out or $n$ bootstrap.
This gives a rule of thumb how to choose $m$ in dependence of $B$.
Note that $B$ determines how good the empirical bootstrap quantile approximates the quantile of the bootstrap distribution and has to be chosen sufficiently large, depending on the quantile of interest. Once this choice has been made, $m$ can be chosen according to  \eqref{eqn:choose_m}.

\paragraph{Computational time.} 
The computational time of the $m$ out of $n$ bootstrap and the $n$ out of $n$ bootstrap is $B$ times the computational time of the estimator $\theta(G_m)$ and $\theta(G_n)$, respectively.
Especially for estimators that are computational intensive, the $m$ out out of $n$ bootstrap with small values for $m$ significantly decreases the computational time in comparison to the $n$ out of $n$ bootstrap.
We will demonstrate this advantage in the following Section \ref{sec:sim} and in Section \ref{sec:sim:addition} of the appendix.

\section{Simulations}
\label{sec:sim}

In this section we evaluate the finite sample performance of the $m$ out of $n$ bootstrap proposed in this paper and compare it with two alternative procedures which have recently been proposed in the literature. 
The first one is the $n$ out of $n$ bootstrap suggested  in \cite{wang2022}. 
These authors propose an additional deconvolution step on the bootstrap estimators $\bar{\theta}_n^\star=\hat{\theta}_n^\star + Y$ to obtain a more accurate estimator of the distribution of $\hat{\theta}$ (note that the distribution of $Y$ is known).
While such a deconvolution is in general possible as well for the $m$ out of $n$ bootstrap, we do not do this here since the obtained simulation results are already good.
This is due to the fact that the additional variance due to privacy in the bootstrap iterations is smaller than in the $n$ out of $n$ bootstrap and therefore the approximation by the bootstrap ecdf is more accurate.
Additionally, by omitting such a deconvolution step, on the one hand we save computational time and on the other hand we circumvent the choice of additional hyperparameters for deconvolution.

The second procedure is BLBQuant which is introduced by  \cite{chadha2024resampling}  and is described in detail in Appendix \ref{sec:appendix-blb}.
The privacy guarantees for BLBQuant are stated in terms of $(\epsilon,\delta)$-DP and can be converted to $\mu$-GDP guarantees by \cref{lem:dp_to_gdp}.
We consider two different scenarios:
\begin{enumerate}
	\item Confidence intervals for the mean of a truncated normal distribution based on synthetic data in \cref{sec:sim:mean}
	\item Confidence intervals for the minimizer of a regularized logistic regression with data from the {\it 2016 Census public use microdata file (PUMF) on individuals} in \cref{sec:sim:log-reg}
\end{enumerate}

The finite sample performance is evaluated based on the empirical coverage of $500$ confidence intervals obtained from the different methods. 
Furthermore, we state the average length of the confidence intervals and the computational time.\\
Additional simulations for confidence intervals for the minimizer of a regularized logistic regression with synthetic data are stated in \cref{sec:sim:addition} in the appendix.

\subsection{Mean of a truncated normal distribution}\label{sec:sim:mean}
We consider a standard normal distribution truncated on the interval $[-5,5]$.
The  $\mu/\sqrt{2}$-GDP  estimate of the mean ($\theta =0$)  is given by  
\[
\bar{\theta}_n = \frac{1}{n}\sum_{i=1}^n X_i + Z,
\]
where $Z\sim N(0,\frac{2\cdot10^2}{n^2\mu^2})$. 
An additional privacy budget of $\mu/\sqrt{2}$ is used for the computation of the bootstrap replications. 
Therefore, by 
\cref{thm:properties_f_DP,cor:CI-correct}, 
the  interval $C_{m,n}(\alpha)$ defined in \eqref{eqn:CI_mn} is an asymptotic $1-2\alpha$ $\mu$-GDP confidence interval for the mean. 
We compare the following simulation set ups for $\alpha=0.05$, $n=500,1000,5000$ and $\mu=0.5, 1$. 
\begin{itemize}
	\item The confidence interval $C_{m,n}(\alpha)$
	defined in \eqref{eqn:CI_mn} which is obtained by the $m$ out of $n$ bootstrap and where $m=m(B)$ is chosen according to the rule \eqref{eqn:choose_m}  for different values $B$.
	\item The confidence interval $C_{n,n}(\alpha)$ defined in \eqref{eqn:CI_mn} (for $m=n$) which is obtained by the $n$ out of $n$ bootstrap where $B=\mu^2n$ as proposed by \cite{wang2022}.
	\item The $(\epsilon,\delta) $-DP confidence intervals obtained by  BLBQuant (see Appendix \ref{sec:appendix-blb} for details), where   $\delta=1/n,1/n^2$ and   $\epsilon$ is chosen such that $\delta(2\epsilon,\mu)=\delta$, see \cref{lem:dp_to_gdp}. The  
	hyperparameters in  BLBQuant are chosen as proposed in \cite{chadha2024resampling} and the mean is computed by the Gauss Mechanism, i.e. 
	$\bar{\theta}_{\epsilon,\delta}= \frac{1}{n}\sum_{i=1}^n X_i + Z,$
	where $Z\sim N(0,10^2/(n^2\tilde{\mu}^2))$ and $\tilde{\mu}$ is such that $\delta(\epsilon,\tilde{\mu})=\delta$ to guarantee $(\epsilon,\delta)$-DP.
\end{itemize}
For each method, we state the empirical coverage and average length based on $500$ confidence intervals and also give the average computation time.\\
In \cref{tab:sim-mn-bootstrap} we display the simulation results for the $m$ out of $n$ (white cells) and $n$ out of $n$ (gray cells) bootstrap. 
We observe that the confidence intervals obtained by the $m$ out of $n$ bootstrap are by a factor $10$ shorter than the ones obtained by $n$ out of $n$ bootstrap while having the desired coverage for $n\geq 1000$. 
Further note that the $n$ out of $n$ bootstrap takes more computational time for a comparable $B$. 
For example, if $\mu =0.5$, $n=5000$ and $B=1000$, the computational time of the $n$ out of $n$ bootstrap is by at least a factor $10$ 
larger than for the $m$ out of $n$ bootstrap and  the coverage probabilities of the $m$ out of $n$ bootstrap are closer to the nominal  level $90\%$ for all choices of $m$ (the lengths of all confidence intervals are comparable).
Note that the rule of thumb \eqref{eqn:choose_m} leads to way smaller choices of $m$ than one would choose in the non-private case.
That the performance of the obtained confidence intervals is still good is due to the fact that there is additional noise introduced to ensure privacy and $m$ is chosen in a way such that the bootstrap estimators $\bar{\theta}_m^\star$ replicate the underlying distribution well. \\
In \cref{tab:sim-blb} we display the confidence intervals obtained from BLBQuant.
We observe that the confidence intervals are too wide and wider than the ones obtained by the $m$ out of $n$ bootstrap.
Furthermore, the computational time is larger which is caused by the resampling and privatization method used in BLBQuant.

\begin{table}[H]
	\centering
	\begin{minipage}{.49\linewidth}
		\begin{tabular}{rrr rrr}
			cov & len & time  & $m$ & $n$ & $B$\\
			\hline \hline
			0.884 & 0.232 & 0.005  &    5 &  500 &  100 \\
			0.886 & 0.236 & 0.023  &    1 &  500 &  500 \\ 
			\rowcolor{lightgray}
			1.000 & 2.301 & 0.010  &  500 &  500 &  125 \\
			\hline
			0.894 & 0.136 & 0.004  &   10 & 1000 &  100 \\
			0.900 & 0.139 & 0.021  &    2 & 1000 &  500 \\
			0.902 & 0.139 & 0.042  &    1 & 1000 & 1000 \\
			\rowcolor{lightgray}
			1.000 & 1.640 & 0.027  & 1000 & 1000 &  250 \\
			\hline
			0.856 & 0.049 & 0.006  &   50 & 5000 &  100 \\ 
			0.890 & 0.050 & 0.023  &   10 & 5000 &  500 \\ 
			0.894 & 0.050 & 0.045  &    5 & 5000 & 1000 \\ 
			\rowcolor{lightgray}
			1.000 & 0.661 & 0.501  & 5000 & 5000 & 1000 \\
		\end{tabular}
		\subcaption{$\mu=0.5$}
	\end{minipage} 
	\begin{minipage}{.49\linewidth}
		\begin{tabular}{rrr rrr}
			cov & len & time  & $m$ & $n$ & $B$\\
			\hline \hline
			0.860 & 0.169 & 0.005 &    5 &  500 &  100 \\ 
			0.878 & 0.172 & 0.022 &    1 &  500 &  500 \\ 
			\rowcolor{lightgray}
			1.000 & 2.325 & 0.035 &  500 &  500 &  500 \\
			\hline 
			0.896 & 0.110 & 0.005 &   10 & 1000 &  100 \\ 
			0.906 & 0.113 & 0.022 &    2 & 1000 &  500 \\ 
			0.868 & 0.113 & 0.043 &    1 & 1000 & 1000 \\ 
			\rowcolor{lightgray}
			1.000 & 1.653 & 0.105 & 1000 & 1000 & 1000 \\ 
			\hline
			0.870 & 0.046 & 0.005 &   50 & 5000 &  100 \\ 
			0.898 & 0.047 & 0.023 &   10 & 5000 &  500 \\
			0.896 & 0.047 & 0.045 &    5 & 5000 & 1000 \\ 
			\rowcolor{lightgray}
			1.000 & 0.332 & 0.487 & 5000 & 5000 & 1000 
		\end{tabular}
		\subcaption{$\mu=1$}
	\end{minipage}
	\caption{Simulated coverage and length of $90\%$ confidence intervals
		for the mean of a truncated normal distribution.
		White cells: the interval $C_{m,n}(\alpha)$ defined in \eqref{eqn:CI_mn},   where $m$ is chosen by \eqref{eqn:choose_m} and  different values of $B$ are considered. 
		Gray cells: confidence intervals obtained by $n$ out of $n$ bootstrap, where $B$ is chosen by the rule of thumb in \cite{wang2022}.}
	\label{tab:sim-mn-bootstrap}
\end{table}

\begin{table}[H]
	\centering
	\begin{minipage}{0.49\linewidth}
		\begin{tabular}{rr rr r l}
			cov &   len &   time &  $n$ & $\epsilon$ & $\delta$\\
			\hline \hline
			0.982 & 0.422 &  4.126 &  500 & 1.234 & $1/n$  \\
			0.984 & 0.480 &  2.651 &  500 & 2.101 &  $1/n^2$  \\ 
			\hline
			0.984 & 0.230 &  6.917 & 1000 & 1.352 & $1/n$  \\
			0.996 & 0.255 &  3.953 & 1000 & 2.254 & $1/n^2$  \\
			\hline
			0.948 & 0.063 & 13.962 & 5000 & 1.600 &  $1/n$  \\
			0.966 & 0.067 &  8.737 & 5000 & 2.579 & $1/n^2$
		\end{tabular}
		\subcaption{$\mu=0.5$}
	\end{minipage}
	\begin{minipage}{0.49\linewidth}
		\begin{tabular}{rr rr r l} 
			cov &   len &   time &  $n$ & $\epsilon$ & $\delta$\\
			\hline \hline
			0.958 & 0.245 &  1.877 &  500  & 2.912 &  $1/n$  \\
			0.976 & 0.267 &  1.422 &  500  & 4.586 & $1/n^2$  \\
			\hline
			0.948 & 0.144 &  2.675 & 1000  & 3.139 &  $1/n$ \\
			0.962 & 0.154 &  1.871 & 1000  & 4.887 & $1/n^2$  \\
			\hline
			0.928 & 0.051 &  5.957 & 5000  & 3.616 &  $1/n$ \\
			0.918 & 0.052 &  3.814 & 5000  & 5.523 & $1/n^2$  \\
		\end{tabular}\subcaption{$\mu=1$}
	\end{minipage}
	\caption{Simulated coverage and length of $90\%$ confidence intervals for the mean of a truncated normal distribution obtained by BLBQuant for $B=500$, where different values of $\delta$ are considered (see Appendix \ref{sec:appendix-blb} for details).}
	\label{tab:sim-blb}
\end{table}

\subsection{Regularized logistic regression}
\label{sec:sim:log-reg}
For the sake of comparison with  \cite{wang2022} we consider regularized logistic regression on the {\it 2016 Census public use microdata file (PUMF) on individuals}  \footnote{Downloaded from \url{https://www150.statcan.gc.ca/n1/pub/98m0001x/index-eng.htm} on August, 4, 2024} which contains social, demographic and economic information about individuals living in Canada. 
Here we consider the variables SHELCO, which gives the monthly shelter cost of a household, and MRKINC, which is the income before transfers and taxes, of the people living in Ontario. 
In a pre-processing step, all unavailable data is removed and both variables are scaled to the interval $[0,1]$. 
The covariate is chosen as $x_i=(1/\sqrt{2},\textnormal{MRKINC}/\sqrt{2})$ and the response as $y_i=1$, if SHELCO$\geq 0.5$ and $y_i=-1$ otherwise.
The estimator of interest is 
\[\hat{\theta}_n = \argmin{\theta\in\mathbb{R}^2} -\frac{1}{n}\sum_{i=1}^n \log\left(\frac{1}{1+\exp(-\theta^\top x_iy_i)}\right) +  \lVert \theta\rVert_2^2,\]
which is privatized by the Gaussian mechanism to ensure $\mu$-GDP
\[\bar{\theta}_n = \hat{\theta}_n + Z,\]
where $Z\sim N(0,1/(n\mu)^2)$. 
In each simulation run, we draw $n=500,1000,5000$ samples with replacement from the data set and compute the confidence interval based on that sample. 
The empirical coverage is evaluated based on the  minimizer over the whole data set.
Again, we compare three different methods:
\begin{itemize}
	\item The confidence interval $C_{m,n}(\alpha)$
	defined in \eqref{eqn:CI_mn} which is obtained by the $m$ out of $n$ bootstrap, where $m=m(B)$ is chosen according to \eqref{eqn:choose_m}  for different values $B$.
	\item The confidence interval $C_{n,n}(\alpha)$ defined in \eqref{eqn:CI_mn} (for $m=n$) which is obtained by the $n$ out of $n$ bootstrap, where $B=\mu^2n$ as proposed by \cite{wang2022}.
	\item The $(\epsilon,\delta) $-DP confidence intervals obtained by  BLBQuant (see Appendix \ref{sec:appendix-blb} for details), where   $\delta=1/n,1/n^2$ and   $\epsilon$ is chosen such that $\delta(2\epsilon,\mu)=\delta$, see \cref{lem:dp_to_gdp}. The  
	hyperparameters in  BLBQuant are chosen as proposed in \cite{chadha2024resampling} and the private minimizer is computed by the Gaussian Mechanism, i.e. 
	$\bar{\theta}_{\epsilon,\delta}= \hat{\theta}_n + Z$,
	where $Z\sim N(0,1/(n\tilde{\mu})^2 )$ and $\tilde{\mu}$ is such that $\delta(\epsilon,\tilde{\mu})=\delta$ to guarantee $(\epsilon,\delta)$-DP. 
\end{itemize}

\begin{table}[h!]
	\centering
	\begin{subtable}{\textwidth} 
		\centering
		\begin{tabular}{rr rr r rrr}
			cov $\theta_0$ & cov $\theta_1$&  len $\theta_0$ &  len $\theta_1$& time & $m$& $B$ & $n$\\ \hline \hline
			0.844  & 0.862  & 0.00124  & 0.000818   & 0.0285& 5 & 100 & 500 \\ 
			0.87  & 0.87 & 0.00126 & 0.000833  & 0.111 & 1 & 500 & 500 \\ 
			\rowcolor{lightgray}
			0.522 & 0.078& 0.000942  & 0.000116 & 0.165 & 500 & 125 & 500 \\ 
			\hline
			0.894 & 0.844  & 0.000549   & 0.000292  & 0.0279  & 10 & 100 & 1000 \\ 
			0.88   & 0.894 & 0.000551  & 0.000299  & 0.114 & 2 & 500 & 1000 \\ 
			0.876  & 0.898 & 0.000544  & 0.000299 & 0.221  & 1 & 1000 & 1000 \\ 
			\rowcolor{lightgray}
			0.716  & 0.118  & 0.000478  & 5.95$\times 10^{-5}$  & 0.58 & 1000 & 250 & 1000 \\ 
			\hline
			0.86  & 0.874 & 9.75$\times 10^{-5}$  & 2.81$\times 10^{-5}$  & 0.0429 & 50 & 100 & 5000 \\ 
			0.892   & 0.89   & 9.93$\times 10^{-5}$  & 2.86$\times 10^{-5}$ & 0.138 & 10 & 500 & 5000 \\ 
			0.884  & 0.884 & 9.9$\times 10^{-5}$  & 2.87$\times 10^{-5}$  & 0.246 & 5 & 1000 & 5000 \\ 
			\rowcolor{lightgray}
			0.866  & 0.306 & 9.64$\times 10^{-5}$  & 1.2$\times 10^{-5}$  & 12.9 & 5000 & 1000 & 5000 
		\end{tabular}
		\caption{$\mu=0.5$}
	\end{subtable}
	\begin{subtable}{\textwidth}
		\centering 
		\begin{tabular}{rr rr r rrr}
			cov $\theta_0$ & cov $\theta_1$&  len $\theta_0$ &  len $\theta_1$& time & $m$& $B$ & $n$\\ \hline \hline
			0.882  & 0.884  & 0.00102  & 0.000418  & 0.0322  & 5 & 100 & 500 \\ 
			0.808  & 0.87  & 0.00096  & 0.000428 & 0.144 & 1 & 500 & 500 \\ 
			\rowcolor{lightgray}
			0.782  & 0.18  & 0.000961  & 0.000118  & 0.744 & 500 & 500 & 500 \\ 
			\hline
			0.88 & 0.854& 0.000485  & 0.000154  & 0.0321  & 10 & 100 & 1000 \\ 
			0.882  & 0.888 & 0.000496 & 0.000157  & 0.149 & 2 & 500 & 1000 \\ 
			0.774 & 0.878 & 0.000438 & 0.000155 & 0.27 & 1 & 1000 & 1000 \\ 
			\rowcolor{lightgray}
			0.82   & 0.258 & 0.000482 & 5.99$\times 10^{-5}$  & 2.8 & 1000 & 1000 & 1000 \\ 
			\hline
			0.892 & 0.86  & 9.49$\times 10^{-5}$  & 1.73$\times 10^{-5}$  & 0.0485 & 50 & 100 & 5000 \\ 
			0.872 & 0.916 & 9.5$\times 10^{-5}$ & 1.74$\times 10^{-5}$ & 0.158  & 10 & 500 & 5000 \\ 
			0.88 & 0.874  & 9.42$\times 10^{-5}$   & 1.71$\times 10^{-5}$ & 0.291 & 5 & 1000 & 5000 \\ 
			\rowcolor{lightgray}
			0.892  & 0.506   & 9.63$\times 10^{-5}$   & 1.2$\times 10^{-5}$   & 13.4 & 5000 & 1000 & 5000
		\end{tabular}
		\subcaption{$\mu=1$}
	\end{subtable}
	
	\caption{Simulated coverage and length of $90\%$ confidence intervals for the parameter $\theta=(\theta_0,\theta_1)$ of regularized logistic regression on the PUMF dataset. White cells: the interval $C_{m,n}(\alpha)$ defined in \eqref{eqn:CI_mn}, where $m$ is chosen by  \eqref{eqn:choose_m} and different values of $B$ are considered. Grey cells: confidence intervals obtained by $n$ out of $n$ bootstrap, where $B$ is chosen by the rule of thumb from \cite{wang2022}.}
	\label{tab:sim_logReg_mn}
\end{table}

\begin{table}[h!]
	\centering
	\begin{tabular}{lllllllll}
		cov $\theta_0$& cov $\theta_1$ & len $\theta_0$ & len $\theta_1$ & t & n & $\mu$ & $\epsilon$ & $\delta$ \\ 
		\hline 
		\hline
		0.916 & 0.95 & 0.0267 & 0.0236 & 102.8094 & 500 & 0.5 & 1.234 & 0.002\\ 
		0.918 & 0.968 & 0.0282 & 0.0279 & 103.1666 & 500 & 0.5 & 2.100 & 0.000004\\ 
		0.896 & 0.958 & 0.0139 & 0.0121 & 300.6152 & 1000 & 0.5 & 1.352 & 0.001\\ 
		0.88 & 0.976 & 0.014 & 0.014 & 300.8687 & 1000 & 0.5 & 2.254 & 0.000001\\ 
		\hline
		0.786 & 0.954 & 0.013 & 0.012 & 103.3567 & 500 & 1 & 2.912 & 0.002\\ 
		0.862 & 0.988 & 0.0161 & 0.016 & 103.175 & 500 & 1 & 4.586 & 0.000004\\ 
		0.724 & 0.94 & 0.007 & 0.006 & 300.921 & 1000 & 1 & 3.139 & 0.001\\ 
		0.78 & 0.982 & 0.008 & 0.008 & 301.3153 & 1000 & 1 & 4.887 & 0.000001\\ 
		\end{tabular}
	\caption{Simulated coverage and length of $90\%$ confidence intervals for the parameter $\theta=(\theta_0,\theta_1)$ of regularized logistic regression on the PUMF dataset obtained by BLBQuant for $B=500$, where different values of $\delta$ are considered (see Appendix \ref{sec:appendix-blb} for details).}
	\label{tab:blb-logReg}
\end{table} 

In \cref{tab:sim_logReg_mn} we display the results for the $m$ out of $n$ (white cells) and $n$ out of $n$ (gray cells) bootstrap.
We see that the average length of the confidence interval obtained by $m$ out of $n$ bootstrap is similar for $B=100,500$ while there are differences regarding the empirical coverage. 
For both choices of $B$, the empirical coverage converges to $90\%$.
For the $n$ out of $n$ bootstrap, we observe that the obtained confidence intervals are too small for $\theta_1$, leading to a too small empirical coverage of the corresponding intervals that only improves slowly for the intercept with an increasing sample size.
It is  also  worthwhile to mention that the standard errors for different choice of $m$
in the $m$ out of $n$ bootstrap 
are comparable (these results are not displayed for the sake of brevity).  For example, if $\mu=1$ and $n=1000$, the standard errors of the length of the confidence intervals for $\theta_1$ are given by $1.12\times10^{-6} $, $ 5.14\times10^{-7}$ and $ 3.75\times 10^{-7}$ for $m=10,2,1$, respectively. 
Moreover, note that the computational time 
is substantially reduced by  the $m$ out of $n$ bootstrap.
For example, if  $n=5000$  the computational time of the $n$ out of $n$ bootstrap is larger than the computational time of the $m$ out of $n$ bootstrap (with the same number of bootstrap replications) by a factor of 30.\\
In \cref{tab:blb-logReg} we display the corresponding simulation results for BLBQuant. 
We observe that the confidence intervals are larger than the ones obtained by the $m$ out of $n$ bootstrap.
For $\mu=1$, the empirical coverage of the confidence intervals for the intercept $\theta_0$ are too small, while the coverage probabilities of the confidence intervals for $\theta_1$ are to large. 
Further note that the computational time of BLBQuant is even larger than the one of the $n$ out of $n$ bootstrap. This is due to the resampling procedure in BLBQuant.

\subsubsection*{Acknowledgments and Disclosure of Funding}
Funded by the Deutsche Forschungsgemeinschaft (DFG, German Research Foundation) under Germany's Excellence Strategy - EXC 2092 CASA - 390781972.
\newpage \mbox{} \newpage
\bibliography{literatur}
\setcounter{section}{0}

\newpage 
\section*{Appendix}

\renewcommand\thesection{\Alph{section}}
\section{Additional Simulation results}
\label{sec:sim:addition}
We state additional simulation results for a regularized logistic regression problem with a $17$-diemnsional covariate chosen as $x_i=(1,x_i^{1},x_i^{2})/\sqrt{17}$, where $x_i^{1}\in \mathbb{R}^8$ consists of 8 independent normal distributions with mean zero and variance 1 truncated on the interval $[0,1]$ and $x_i^{2}\in \mathbb{R}^8$ consists of 8 independently uniformly distributed random variables on the interval $[0,1]$.
The response $y_i$ is randomly distributed on $\{-1,1\}$ with 
$$
P(y_i=1)=\frac{\exp( \theta^\top  x_i)}{1+\exp(\theta^\top  x_i)}
$$
where 
$\theta =  (0 , 5 , 5 , 5 , 5 , 5 , 5 , 5 , 5, -5 ,-5 ,-5 ,-5, -5 ,-5, -5, -5)^\top $.
The estimator of interest is 
\[\hat{\theta}_n = \argmin{\theta\in\mathbb{R}^{17}} -\frac{1}{n}\sum_{i=1}^n \log\left(\frac{1}{1+\exp(-\theta^\top x_iy_i)}\right) +  \lVert \theta\rVert_2^2,\]
which is privatized by the Gaussian mechanism to ensure $\mu$-GDP
\[\bar{\theta}_n = \hat{\theta}_n + Z,\]
where $Z\sim N(0,1/(n\mu)^2)$. 
In a first step, we draw a sample of size $n=10^6$ and compute the minimizer over that dataset as a reference for computing the empirical coverage of the obtained confidence intervals.
In each simulation run, we will draw $n=500,1000,5000,10000$ samples with replacement from this data set and compute the confidence interval based on that sample. 
We compare two different methods: 
\begin{itemize}
	\item The confidence interval $C_{m,n}(\alpha)$
	defined in \eqref{eqn:CI_mn} which is obtained by the $m$ out of $n$ bootstrap, where $m=m(B)$ is chosen according to \eqref{eqn:choose_m}  for different values $B$.
	\item The confidence interval $C_{n,n}(\alpha)$ defined in \eqref{eqn:CI_mn} (for $m=n$) which is obtained by the $n$ out of $n$ bootstrap, where $B=\mu^2n$ as proposed by \cite{wang2022}.
\end{itemize}
We exemplarily display the simulation results for the intercept $\theta_0$ and the 8th and 10th coordinate $\theta_8,\theta_{10}$ in \cref{tab:addreg}.

The empirical coverage of the confidence intervals obtained from the $m$ out of $n$ bootstrap is closer to the target level of $90\%$ than the confidence intervals computed by the $n$ out of $n$ bootstrap, even for $n=10 000$.
The lengths of these confidence intervals from are much shorter than the ones obtained from the $m$ out of $n$ bootstrap for $n\leq 1000$, but their coverage probabilities are meaningless.
We also note that the computational time is shorter for the $m$ out of $n$ bootstrap. For example, if  $n=10 000$, it takes  at most $0.3$ seconds for the $m$ out of $n$ bootstrap, while  the $n$ out of $n$ about $40$ seconds. This is a factor of $130$ and will be even larger with increasing sample size.

\begin{table}[]
	\centering
	\begin{tabular}{rrr rrr r rrr r}
		cov $\theta_0$ & cov $\theta_8$ & cov $\theta_{10}$ & len $\theta_0$ & len $\theta_8$& len $\theta_{10}$ & time & n & m & B & $\mu$\\ 
		\hline
		\hline
		0.816 & 0.840 & 0.830 & 0.00089 & 0.00084 & 0.00084 & 0.03948 & 500 & 5 & 100 & 0.5\\ 
		0.876 & 0.876 & 0.872 & 0.00091 & 0.00085 & 0.00085 & 0.13652 & 500 & 1 & 500 & 0.5\\ 
		\rowcolor{lightgray}
		0.278 & 0.15 & 0.158 & 0.00037 & 0.0002 & 0.0002 & 0.31151 & 500 & 500 & 125 & 0.5\\ 
		\hline 
		0.826 & 0.852 & 0.858 & 0.00034 & 0.0003 & 0.0003 & 0.04534 & 1000 & 10 & 100 & 0.5\\ 
		0.894 & 0.89 & 0.87 & 0.00035 & 0.00031 & 0.00031 & 0.13819 & 1000 & 2 & 500 & 0.5\\ 
		0.858 & 0.886 & 0.9 & 0.00035 & 0.00031 & 0.00031 & 0.24442 & 1000 & 1 & 1000 & 0.5\\ 
		\rowcolor{lightgray}
		0.406 & 0.216 & 0.242 & 0.00019 & 0.0001 & 0.0001 & 1.07549 & 1000 & 1000 & 250 & 0.5\\ 
		\hline 
		0.896 & 0.876 & 0.856 & 0.00004 & 0.00003 & 0.00003 & 0.09716 & 5000 & 50 & 100 & 0.5\\ 
		0.888 & 0.886 & 0.874 & 0.00005 & 0.00003 & 0.00003 & 0.17455 & 5000 & 10 & 500 & 0.5\\ 
		0.902 & 0.87 & 0.892 & 0.00005 & 0.00003 & 0.00003 & 0.3117 & 5000 & 5 & 1000 & 0.5\\ 
		\rowcolor{lightgray}
		0.638 & 0.482 & 0.472 & 0.00004 & 0.00002 & 0.00002 & 19.61842 & 5000 & 5000 & 1000 & 0.5\\ 
		\hline
		0.876 & 0.87 & 0.892 & 0.00002 & 0.00001 & 0.00001 & 0.08891 & 10000 & 100 & 100 & 0.5\\ 
		0.876 & 0.888 & 0.894 & 0.00002 & 0.00001 & 0.00001 & 0.20381 & 10000 & 20 & 500 & 0.5\\ 
		0.890 & 0.896 & 0.91 & 0.00002 & 0.00001 & 0.00001 & 0.29243 & 10000 & 10 & 1000 & 0.5\\ 
		\rowcolor{lightgray}
		0.768 & 0.584 & 0.588 & 0.00002 & 0.00001 & 0.00001 & 40.81514 & 10000 & 10000 & 1000 & 0.5\\ 
		\hline
		\hline
		\hline 
		0.850 & 0.84 & 0.844 & 0.00055 & 0.00045 & 0.00045 & 0.03501 & 500 & 5 & 100 & 1\\ 
		0.886 & 0.866 & 0.892 & 0.00055 & 0.00046 & 0.00046 & 0.12462 & 500 & 1 & 500 & 1\\ 
		\rowcolor{lightgray}
		0.53 & 0.324 & 0.322 & 0.00038 & 0.00021 & 0.00021 & 1.20586 & 500 & 500 & 500 & 1\\ 
		\hline
		0.888 & 0.866 & 0.88 & 0.00023 & 0.00018 & 0.00018 & 0.04303 & 1000 & 10 & 100 & 1\\ 
		0.868 & 0.892 & 0.892 & 0.00024 & 0.00018 & 0.00018 & 0.13252 & 1000 & 2 & 500 & 1\\ 
		0.884 & 0.882 & 0.898 & 0.00023 & 0.00018 & 0.00018 & 0.25739 & 1000 & 1 & 1000 & 1\\ 
		\rowcolor{lightgray}
		0.658 & 0.404 & 0.42 & 0.00019 & 0.0001 & 0.0001 & 4.11911 & 1000 & 1000 & 1000 & 1\\ 
		\hline
		0.900 & 0.886 & 0.89 & 0.00004 & 0.00002 & 0.00002 & 0.05685 & 5000 & 50 & 100 & 1\\ 
		0.906 & 0.9 & 0.914 & 0.00004 & 0.00002 & 0.00002 & 0.1465 & 5000 & 10 & 500 & 1\\ 
		0.902 & 0.894 & 0.902 & 0.00004 & 0.00002 & 0.00002 & 0.26353 & 5000 & 5 & 1000 & 1\\ 
		\rowcolor{lightgray}
		0.83 & 0.686 & 0.682 & 0.00004 & 0.00002 & 0.00002 & 20.6991 & 5000 & 5000 & 1000 & 1\\ 
		\hline
		0.888 & 0.9 & 0.866 & 0.00002 & 0.00001 & 0.00001 & 0.08654 & 10000 & 100 & 100 & 1\\ 
		0.894 & 0.91 & 0.91 & 0.00002 & 0.00001 & 0.00001 & 0.17674 & 10000 & 20 & 500 & 1\\ 
		0.89 & 0.888 & 0.874 & 0.00002 & 0.00001 & 0.00001 & 0.28269 & 10000 & 10 & 1000 & 1\\ 
		\rowcolor{lightgray}
		0.864 & 0.772 & 0.792 & 0.00002 & 0.00001 & 0.00001 & 40.32045 & 10000 & 10000 & 1000 & 1\\ 
	\end{tabular}

	\caption{Simulated coverage and length of $90\%$ confidence intervals for the intercept $\theta_0$ and the $8$-th and $10$-th coordinate of a regularized logistic regression with $\theta\in \mathbb{R}^{17}$ on synthetic data. 
		White cells: the confidence interval $C_{m,n}(\alpha)$ defined in \eqref{eqn:CI_mn}, where $m$ according to  the rule  \eqref{eqn:choose_m}, and different values of $B$ are considered.
		Gray cells: confidence intervals obtained by $n$ out of $n$ bootstrap, where $B$ is chosen by the rule of thumb in \cite{wang2022}.}
	\label{tab:addreg}
\end{table}

\section{Proofs}
\subsection{Proof of \cref{thm:asymp_privacy_bootstrap_mn}}
\label{sec:m_n_is_dp}
We start by stating  several results  and definitions from  \cite{dong2022gaussian} and \cite{wang2022} that are used in the proof  of \cref{thm:asymp_privacy_bootstrap_mn}.  We introduce the notations 
\begin{align*}
	kl(f) &:= -\int_0^1 \log |f'(x)|dx\\
	\kappa_2&:=\int_0^1\log ^2|f'(x)|dx\\
	\kappa_3 &:= \int_0^1|\log|f'(x)||^3dx\\
	\bar{\kappa}_3 &:= \int_0^1|\log|f'(x)|+kl(f)|^3dx.
\end{align*}
and state the  following CLT-type theorem.
\begin{thm}[Theorem 6 \cite{dong2022gaussian}] \label{thm:asyptotic_composition}
	Let $\{f_{ni}: 1\leq i\leq n\}_{n=1}^\infty$ be a triangular array of symmetric trade-off functions and assume  that there exists the  constants $K\geq 0$ and $s>0$ such that 
	\begin{enumerate}
		\item $\lim\limits_{n\to \infty}\sum\limits_{i=1}^n kl(f_{ni})= K$
		\item $\lim\limits_{n\to \infty}\max\limits_{1\leq i \leq n} kl(f_{ni})=  0$
		\item $\lim\limits_{n\to \infty}\sum\limits_{i=1}^n \kappa_2(f_{ni})=  s^2$
		\item $\lim\limits_{n\to \infty}\sum\limits_{i=1}^n \kappa_3(f_{ni})= 0$.
	\end{enumerate}
	Then, we have
	\[\lim_{n\to\infty} f_{n1}\otimes \ldots \otimes f_{nn} (\alpha)=G_{2K/s}(\alpha)\]
	uniformly for all $\alpha \in [0,1]$.
\end{thm}

For a vector $\mathbf{f} =(f_1,\ldots,f_k)$ of  trade-off functions and a vector $\mathbf{p} =(p_1,\ldots , p_k) $ of probabilities with $\sum_{i=1}^k p_i =1$  define the {\it mixture of trade-off functions} $mix(\mathbf{p},\mathbf{f}) : [0,1] \to  [0,1]  $ as
\begin{equation}
	\label{eqn:def-mix}
	mix(\mathbf{p},\mathbf{f} )(\alpha):=\begin{cases}
		\sum_{i=1}^k p_i f_i(0) & \text{ , } \alpha=0\\
		0 & \text{ , } \alpha=1\\
		\sum_{i=1}^k p_i f_i(\alpha_i) & \text{ , otherwise}
	\end{cases},
\end{equation}
where $\alpha_1, \ldots , \alpha_k $ are such that $\sum_{i=1}^k p_i\alpha_i=\alpha$ and such that there exists a constant $C$ with $C\in \partial f_i(\alpha_i)$ for $i=1,\ldots, k$. 
By Lemma 28 of \cite{wang2022}, 
$mix(\mathbf{p},\mathbf{f} )$ is a well-defined trade-off function.

To compute asymptotic privacy guarantees for the $m$ out of $n$ bootstrap, let  
$i_1, \ldots  , i_m$
denote independent random variables uniformly distributed on the index set  $\{1,\ldots,n\}$ and define the  randomized mapping
\begin{align*}
	boot_{m,n}:
	\begin{cases}
		&  \mathcal{X}^n\to \mathcal{X}^m\\
		& (x_1,\ldots,x_n)\mapsto (x_{i_1},\ldots,x_{i_m})  
	\end{cases} ~~. 
\end{align*}
Denote the probability, that a given entry is $i$ times included in $boot_{m,n}(x_1,\ldots,x_n)$, by
\begin{equation}\label{eq:bootstrap_mn_prob}
	p_{m,i}=\binom{m}{i}\left(\frac{1}{n}\right)^i \left(1-\frac{1}{n}\right)^{m-i}.
\end{equation}
Further denote the convex conjugate of a trade-off function $f$ by $f^\star(y)=\sup_{-\infty<x<\infty} xy-f(x)$ and for $0\leq p\leq 1$ denote 
$f_p(x):= pf(x) - (1-p)(1-x)$ 
and 
\begin{equation}\label{eqn:def-C_p}
	C_p(f):= min\{f_p,f_p^{-1}\}^{\star\star}.
\end{equation}
With this, we can apply Theorem 8 of \cite{wang2022} to obtain the following privacy guarantee for one bootstrap estimator, from which the privacy guarantee for $\bar{\theta}_m^\star$ as stated in \eqref{det101} follows with $M=M_{Gauss}$ and $f_i=G_{i\mu_B^\star}$. 

\begin{thm} \label{thm:privacy_bootstrap_mn}
	Let $f_i$ be a symmetric trade-off function for $i=1,\ldots, m$. For a mechanism $M$ satisfying $f_i$-DP with group size $i$, $M\circ boot_{m,n}$ satisfies $f_{M\circ boot_{m,n}}$-DP where 
	\[f_{M\circ boot_{m,n}}=C_{1-p_{m,0}} (mix(\mathbf{p},\mathbf{f})),\]
	where $\mathbf{f}=(f_1,\ldots,f_m)$, $\mathbf{p}=(\frac{p_{m,1}}{1-p_{m,0}},\ldots,\frac{p_{m,m}}{1-p_{m,0}})$ and $p_{m,i}$ as in \eqref{eq:bootstrap_mn_prob}.
\end{thm}
\begin{proof}
	Let $D_1=(x_1,\ldots,x_n)$, $D_2=(x_1',x_2,\ldots,x_n)$ two neighboring databases that, without loss of generality,  differ in the first entry.
	We will give a lower bound $f_{min}(\alpha)$  for the trade-off function $T(M\circ boot_{m,n}(D_1),M\circ boot_{m,n}(D_2))(\alpha)$, which coincides with $f_{M\circ boot_{m,n}}$.\\
	In a first step only consider the events, where the first element is sampled at least once in the bootstrap iterations. 
	Denote this randomized mapping by $boot_{m,n}^>$, i.e.
	\[boot_{m,n}^> = \begin{cases}
		\mathcal{X}^n\to \mathcal{X}^m & \\
		(x_1,\ldots,x_n)\mapsto (x_{i_1},\ldots,x_{i_m})& \textnormal{such that } |\{j : i_j=1\}| \geq 1.
	\end{cases}\]
	The probability that $x_1$ is sampled exactly $i$ times by $boot_{m,n}^>$ is given by $\frac{p_{m,i}}{1-p_{m,0}}$, where  $p_{m,i}$ is defined  in \eqref{eq:bootstrap_mn_prob}. 
	Therefore, by Theorem 8 of \cite{wang2022}, $M\circ boot_{m,n}^>$ is $f_>$-DP, where 
	$$
	f_>=mix(\mathbf{f},\mathbf{p}) 
	$$ and 
	$$\mathbf{f}=(f_1,\ldots,f_m)~, ~~~ 
	\mathbf{p}=(\frac{p_{m,1}}{1-p_{m,0}},\ldots,\frac{p_{m,m}}{1-p_{m,0}}).$$
	Denote by $boot_{m,n}^0$ the randomized mechanism where the first element is not included in the bootstrap sample, i.e.
	\[boot_{m,n}^0 = \begin{cases}
		\mathcal{X}^n\to \mathcal{X}^m & \\
		(x_1,\ldots,x_n)\mapsto (x_{i_1},\ldots,x_{i_m})& \textnormal{such that } i_j \neq 1 \textnormal{ for all } j=1,\ldots,m
	\end{cases}\]
	and note that $M\circ boot_{m,n}^0$ is $f_0$-DP, where $f_0(x)=1-x$.\\
	Next,  note that the Type I error of a decision rule $\psi$ between the hypotheses $M\circ boot_{m,n}(D_1)$ vs. $M\circ boot_{m,n}(D_2)$ is given by 
	\[\alpha=\E_{M\circ boot_{m,n}(D_1)}[\psi] =p_{m,0} \alpha_0 + (1-p_{m,0})\alpha_>,\]
	where $\alpha_0=\E_{M\circ boot_{m,n}^0(D_1)}[\psi]$, $\alpha_>=\E_{M\circ boot_{m,n}^>(D_1)}[\psi]$ and therefore  
	\begin{align}
		\label{det201}
		\alpha_>
		=\frac{\alpha-p_{m,0}\alpha_0}{1-p_{m,0}}.
	\end{align}
	Denoting $\beta_0=\E_{M\circ boot_{m,n}^0(D_2)}[1-\psi]=1-\alpha_0$ and  
	$\beta_>=\E_{M\circ boot_{m,n}^>(D_2)}[1-\psi]$ we obtain by 
	\begin{align} \label{eqn:type2}
		\beta=\E_{M\circ boot_{m,n}(D_2)}[1-\psi] =p_{m,0} \beta_0 + (1-p_{m,0})\beta_>
	\end{align}
	a similar expression for the Type II error.
	In a third step, with the same arguments as in \cite{wang2022}, it can be shown that $\beta_0=1-\alpha_0\geq f_>(\alpha_>)$ and $\beta_>\geq f_>(\alpha_0)$.
	Therefore, a lower bound for the smallest achievable Type II error for a given Type I error for distinguishing $M\circ boot_{m,n}(D_1)$ vs. $M\circ boot_{m,n}(D_2)$ is obtained by minimizing \eqref{eqn:type2} with respect to $\beta_0,\beta_>$ under the before mentioned constraints.
	More precisely, for a given $\alpha$, denote by $f_{min}(\alpha)$ the minimum of
	\begin{align}
		\label{det202}
		&
		p_{m,0}(1-\alpha_0) + (1-p_{m,0})\beta_>
	\end{align}
	with respect to $\alpha_0$ and 
	$\beta_>$
	subject to 
	\begin{align*}
		\begin{array}{ccccc}
			0 & \leq & \alpha_0  & \leq  & 1-f_>\left(\frac{\alpha-p_{m,0}\alpha_0}{1-p_{m,0}}\right)\\
			f_>\left(\frac{\alpha-p_{m,0}\alpha_0}{1-p_{m,0}}\right) & \leq & \beta_> & \leq & 1\\
			f_>(\alpha_0) & \leq & \beta_> &&
		\end{array}
	\end{align*}
	Then it holds that $f_{min}(\alpha)\leq \beta$ and therefore $T(M\circ boot_{m,n}(D_1),M\circ boot_{m,n}(D_2))\geq f_{min}$.
	Solving this optimization problem using the Karush–Kuhn–Tucker theorem leads for \eqref{det202} the representation
	\[f_{min}(\alpha)=\begin{cases}
		p_{m,0} + (1-p_{m,0})f_>(0) &\textnormal{ , } \alpha= 0\\
		p_{m,0}(1-\alpha) + (1-p_{m,0}) f_>(\alpha) &\textnormal{ , } 0<\alpha\leq\bar{\alpha}\\
		p_{m,0} - \alpha +2(1-p_{m,0})\bar{\alpha} &\textnormal{ , } \bar{\alpha}<  \alpha \leq (1-2p_{m,0})\bar{\alpha} + p_{m,0}\\
		1-\alpha_0^\star &\textnormal{ , }(1-2p_{m,0})\bar{\alpha} + p_{m,0}<\alpha < p_{m,0} + (1-p_{m,0})f_>(0)\\
		0 &\textnormal{ , }  p_{m,0} + (1-p_{m,0})f_>(0) \leq \alpha \\
	\end{cases},\]
	where $f_>(\bar{\alpha})=\bar{\alpha}$ and $p_{m,0}\alpha_0^\star + (1-p_{m,0})f_>(1-\alpha_0^\star) =\alpha$. 
	Using similar arguments as in \cite{wang2022}, it can be shown that $f_{min}$ coincides with $C_{1-p_{m,0}}(f_>)$, which concludes the proof.
\end{proof}

With this, we will now state the proof of the theorem.
\begin{proof}[Proof of \cref{thm:asymp_privacy_bootstrap_mn}]
	We use similar arguments as given  in the proof of Theorem 12 of \cite{wang2022}. 
	In a first step, note that a mechanism satisfying $\mu_B^\star$-GDP satisfies $i\mu_B^\star$-GDP for group size $i$. 
	Therefore, by \cref{thm:privacy_bootstrap_mn}, the $b$-th bootstrap sample is $f_{B,b,boot_{m,n}}=C_{1-p_{m,0}}(f_>)$-DP, where 
	\[f_> = mix\left((G_{\mu_B^\star}, G_{2\mu_B^\star},\ldots,G_{m\mu_B^\star}),\left(\frac{p_{m,1}}{1-p_{m,0}},\ldots,\frac{p_{m,m}}{1-p_{m,0}}\right)\right)\]
	and $G_{k\mu^*_B}$  is the trade-off function  in \eqref{det102}.  
	By  Lemma 29 in \cite{wang2022} it follows   that $f_>$ is the trade-off function for testing 
	\[H_0\text{: } X\sim \sum_{i=1}^m \frac{p_{m,i}}{1-p_{m,0}}N\left(-\frac{i^2\mu_B^{\star2}}{2},i^2\mu_B{\star2}\right) \text{ vs. } 
	H_1\text{: }  X\sim \sum_{i=1}^m \frac{p_{m,i}}{1-p_{m,0}}N\left(\frac{i^2\mu_B^{\star2}}{2},i^2\mu_B^{\star2}\right).\]
	This representation together with Lemma F.2 and F.3 in \cite{dong2022gaussian} can be used to obtain the following representation for the quantities in \cref{thm:asyptotic_composition}:
	\begin{align*}
		kl(f_>)&=\sum_{i=1}^m \frac{p_{m,i}}{1-p_{m,0}}\frac{i^2\mu_B^{\star2}}{2} = \frac{m}{n}\frac{1+m/n - 1/n}{2(1-p_{m,0})}\mu_B^{\star2}\\
		\kappa_2(f_>)&=\sum_{i=1}^m \frac{p_{m,i}}{1-p_{m,0}}\left(\frac{i^4\mu_B^{\star4}}{4} + i^2\mu_B^{\star2}\right)= \frac{m}{n}\frac{1+m/n - 1/n}{1-p_{m,0}}\mu_B^{\star2} + \Theta(\mu_B^{\star4})\\
		\kappa_3(f_>) &\in \Theta(\mu_B^{\star3}). 
	\end{align*}
	Further, following same arguments as \cite{wang2022}, it holds that
	\begin{align*}
		\lim_{\mu_B\to 0}kl(C_{1-p_{m,0}}(f_>)) &= (1-p_{m,0})^2\lim_{\mu_B\to 0} kl(f_>)\\
		\lim_{\mu_B\to 0}\kappa_2(C_{1-p_{m,0}}(f_>)) &= (1-p_{m,0})^2\lim_{\mu_B\to 0}\kappa_2(f_>)\\
		\lim_{\mu_B\to 0} \kappa_3(C_{1-p_{m,0}}(f_>)) &= (1-p_{m,0})^3\lim_{\mu_B\to 0}\kappa_3(f_>).
	\end{align*}
	Noting that $p_{m,0}=(1-1/n)^m$ and with the assumption 
	$$\lim_{B\to\infty} \mu_B^{\star2} B\left(1-\left(1-\frac{1}{n}\right)^m\right)\frac{m}{n}\left(1+\frac{m-1}{n}\right)\to \mu^2 ,
	$$
	we obtain that 
	\begin{align*}
		K=\lim_{B\to\infty} \sum_{i=1}^B kl(f_{Bi,boot}) &= (1-p_{m,0})^2 \lim_{B\to\infty} B kl(f_>) = \mu^2/2\\
		\lim_{B\to\infty } \max_{1\leq i \leq B} kl(f_{Bi,boot}) &= (1-p_{m,0})^2\lim_{B\to\infty }  kl(f_>) = 0\\
		s^2=\lim_{B\to\infty} \sum_{i=1}^B \kappa_2(f_{Bi,boot}) &= (1-p_{m,0})^2 \lim_{B\to\infty} B \kappa_2(f_>) = \mu^2\\
		\lim_{B\to\infty} \sum_{i=1}^B \kappa_3(f_{Bi,boot}) &= (1-p_{m,0})^3 \lim_{B\to\infty} B \kappa_3(f_>) = 0
	\end{align*}
	and therefore $2K/s= \mu$ .
\end{proof}
\subsection{Proof of \cref{thm:asymp_consistency_bootstrap_mn}}
It holds that 
\begin{align*}
	&\sup_x |P(\sqrt{n}(\bar{\theta}_n-\theta)\leq x) - \frac{1}{B}\sum_{i=1}^B \mathbb{I}\{\sqrt{m}(\bar{\theta}_m^{\star(i)}-\bar{\theta}_n)\leq x\}|\\
	\leq & \sup_x |P(\sqrt{n}(\hat{\theta}_n-\theta)\leq x) - P^\star(\sqrt{m}(\hat{\theta}_m^\star-\hat{\theta}_n)\leq x|X_1,\ldots,X_n)|\\
	&+ \sup_x |P(\sqrt{n}(\bar{\theta}_n-\theta)\leq x) -P(\sqrt{n}(\hat{\theta}_n-\theta)\leq x)|\\
	&+\sup_x |P(\sqrt{m}(\bar{\theta}_m^\star-\bar{\theta}_n)\leq x|X_1,\ldots,X_n,Z) - P(\sqrt{m}(\hat{\theta}_m^\star-\hat{\theta}_n)\leq x|X_1,\ldots,X_n)|\\
	&+\sup_x|\frac{1}{B}\sum_{b=1}^B \mathbb{I}\{\sqrt{m}(\bar{\theta}_{m,b}^{\star}-\bar{\theta}_n)\leq x\}-P(\sqrt{m}(\bar{\theta}_m^\star-\bar{\theta}_n)\leq x|X_1,\ldots,X_n,Z)|.
\end{align*}
The first term converges to zero by the assumption on the  consistency of the non-private $m$ out of $n$ bootstrap.
Further note that $\bar{\theta}_n=\hat{\theta}_n + Z$ , see \eqref{eqn:gdp_esti}, and $\bar{\theta}_m^\star=\hat{\theta}_m^\star +Y$, see \eqref{eqn:bootstrap_private}, and for the quantities  $\sqrt{m}Z$, $\sqrt{n}Z$ and $\sqrt{m}Y$ it holds that each 
converge to zero in probability by the assumption on the sensitivity of the estimator, on $B$ and on $m$.
Therefore, the second term converges to zero as well.
The third term converges in probability to zero since $\sqrt{m}Z$ and $\sqrt{m}Y$ converge to zero in probability and therefore both terms converge to the same limit in probability.
The last term converges in probability to zero as well, since its conditional expectation can be bounded in the following way:
\begin{align*}
	&\E[\sup_x|\frac{1}{B}\sum_{b=1}^B \mathbb{I}\{\sqrt{m}(\bar{\theta}_{m,b}^{\star}-\bar{\theta}_n)\leq x\}-P(\sqrt{m}(\bar{\theta}_m^\star-\bar{\theta}_n)\leq x|X_1,\ldots,X_n,Z)||X_1,\ldots,X_n,Z] \\
	= & \int_0^\infty P(\sup_x|\frac{1}{B}\sum_{b=1}^B \mathbb{I}\{\sqrt{m}(\bar{\theta}_{m,b}^{\star}-\bar{\theta}_n)\leq x\}-P(\sqrt{m}(\bar{\theta}_m^\star-\bar{\theta}_n)\leq x|X_1,\ldots,X_n,Z)|>z|X_1,\ldots,X_n,Z)dz\\
	\leq& \int_0^\infty 2\exp(-2Bz^2) dz = \sqrt{\frac{\pi}{2B}}
\end{align*}
which converges to zero. Therefore the assertion of the theorem follows. 
\newpage
\section{Details on BLBQuant} \label{sec:appendix-blb}
For completeness, we state here the algorithm BLBQuant from \cite{chadha2024resampling}. The parameters are chosen in the following way:
\begin{align*}
	s&=\min\left\{\max\left\{2,\left\lfloor 10 \frac{\log(n)}{\epsilon}\right\rfloor \right\} , n\right\}\\
	I_t &=[-t\sqrt{n},t\sqrt{n}] \textnormal{ , for } t=1,\ldots,T\\
	T &= \left\lceil 5 b_\sigma \sqrt{n}\right\rceil,
\end{align*}
where $b_\sigma$ is an upper bound for the standard derivation of $\sqrt{n}(\bar{\theta}(G_n)-\theta(G))$. For the truncated normal distribution, we chose $b_\sigma =5$ and for logistic regression we chose $b_\sigma=1$. The number of bootstrap iterations is chosen as $B=300$.\\
A $(2\epsilon,\delta)$-DP symmetric $1-\alpha$ confidence interval for $\theta(G)$ is then obtained by combining the $\epsilon$-DP output of BLBQuant with the $(\epsilon,\delta)$-DP estimator $\bar{\theta}(G_n)$, i.e.
\[ CI_{BLB}=[\bar{\theta}(G_n) -\bar{t}/\sqrt{n}, \bar{\theta}(G_n) +\bar{t}/\sqrt{n}]. \]

For a comparison with $\mu$-GDP, we choose $(2\epsilon,\delta)$ in the following way: Depending on the sample size take $\delta=\frac{1}{n}, \frac{1}{n^2}$. Then choose $\epsilon$ (depending on $\mu$), such that $\delta(2\epsilon)=\delta$, where $\delta(\epsilon)$ as in \cref{lem:dp_to_gdp}.
\begin{algorithm}[H]
	\begin{algorithmic}
		\State {\bf Input:} Sample $G_n$, estimator $\theta(\cdot)$ and its $(\epsilon,\delta)$-DP version $\bar{\theta}(\cdot)$, number of bootstrap iterations $B$, number of bags $s$, privacy parameter $(\epsilon, \delta)$, sequence of sets of interest $I_{1:T}$, level $\alpha$
		\State {\bf Output:} a  $\epsilon$-DP $1-\alpha$ confidence set for $\sqrt{n}(\theta(G)-\bar{\theta}(G_n))$
		\State Draw $s$ disjoint subsamples $G_m^{sub(1)},\ldots,G_m^{sub(s)} $ of size $m=\lfloor n/s \rfloor$.
		\For{$t=1,\ldots,T$}
		\For{$i=1,\ldots,s$}
		\State $\hat{\theta}_m \leftarrow \theta(G_m^{sub(i)})$
		\For{$j=1,\ldots,B$}
		\State Draw resample $G_n^\star$ of size $n$ iid with replacement from $G_m^{sub(i)}$
		\State $\theta_j^\star\leftarrow \bar{\theta}(G_n^\star)$
		\EndFor
		\State $\hat{p}_i(t)\leftarrow\frac{1}{B}\sum_{j=1}^B\mathbb{I}\{\sqrt{n}(\hat{\theta}_m-\theta_j^\star)\in I_t\}$
		\EndFor
		\State $\hat{p}(t)=(\hat{p}_1(t),\ldots,\hat{p}_s(t))$
		\EndFor
		\State Draw $\xi_0\sim Lap(s/2,2/\epsilon)$ and $\xi_i\sim Lap(0,4/\epsilon)$ iid for $i=1,\ldots T$.
		\State $\bar{t}=\mathbf{AboveThr}(\{\hat{p}(1),\ldots,\hat{p}(T)\},1-\alpha,(\xi_0,\xi_1,\ldots,\xi_T))$
		\If{$\bar{t}=\perp$}
		\State \Return $(-\infty, \infty)$
		\EndIf
		\State \Return $I_{\bar{t}}$
	\end{algorithmic}
	\caption{{\bf BLBQuant} from \cite{chadha2024resampling}}
	\label{algo:BLBQuant}
\end{algorithm}

\begin{algorithm}[H]
	\begin{algorithmic}
		\State {\bf Input: } data $y(1),\ldots,y(T)\in\mathbb{R}^s$, threshold $\tau\in\mathbb{R}$, Noise $\xi_0,\xi_1,\ldots,\xi_T\in \mathbb{R}$ 
		\State {\bf Output: } With high probability and $\epsilon$-DP: $\min \{t : \textnormal{med}(y(t))=\textnormal{med}(y_1(t),\ldots,y_s(t))\geq \tau\}$  
		\For{$t=1,\ldots,T$}
		\State Let $y_{(1)}(t)\leq\ldots\leq y_{(s)}(t)$ be the order statistic of $y_1(t),\ldots,y_s(t)$.\\
		\State $\hat{v} \leftarrow \begin{cases}
			y_{(\lfloor\xi_0 +\xi_t \rfloor)}(t) & \textnormal{ , if } 1\leq \xi_0 +\xi_t \leq s\\
			-\infty \mathbb{I}\{\xi_0 +\xi_t<1\} + \infty \mathbb{I}\{\xi_0 +\xi_t>s\} &\textnormal{ , otherwise}
		\end{cases}$\\
		\If{$\hat{v}\geq \tau$}
		\State \Return t
		\EndIf
		\EndFor
		\State \Return $\perp$
	\end{algorithmic}
	\caption{{\bf AboveThr} from \cite{chadha2024resampling}}
	\label{algo:AboveThr}
\end{algorithm}

\end{document}